\documentclass{article}

\usepackage{microtype}
\usepackage{graphicx}
\usepackage{subfigure}
\usepackage{booktabs} 

\usepackage{mathtools} 

\usepackage{hyperref}
\hypersetup{colorlinks,
            linkcolor=blue,
            citecolor=blue,
            urlcolor=magenta,
            linktocpage,
            plainpages=false}

\usepackage[utf8]{inputenc} 
\usepackage[T1]{fontenc}    
\usepackage{hyperref}       
\usepackage{url}            
\usepackage{booktabs}       
\usepackage{amsfonts}       
\usepackage{nicefrac}       
\usepackage{microtype}      
\usepackage{xcolor}         

\usepackage{bm}
\usepackage{bbm}
\usepackage{comment}         
\usepackage{multicol}
\usepackage{multirow}
\usepackage{caption}
\usepackage{natbib}
\usepackage{braket}

\usepackage{amsmath,amssymb,amsthm}

\usepackage{color}

\usepackage{graphicx}

\DeclareMathOperator*{\argmin}{arg\,min}

\usepackage{times}

\newcommand{\pa}{\mathrm{\pa}}

\newcommand{\RN}[1]{%
  \textup{\uppercase\expandafter{\romannumeral#1}}%
}

\usepackage{xcolor}
\newcount\Comments  
\newcommand{\kibitz}[2]{\ifnum\Comments=1\textcolor{#1}{#2}\fi}

\usepackage{algorithm}
\usepackage{algorithmic}

\usepackage{authblk}
\allowdisplaybreaks

\usepackage{arxiv}

\usepackage{amsmath}
\usepackage{amssymb}
\usepackage{mathtools}
\usepackage{amsthm}

\usepackage[capitalize,noabbrev]{cleveref}

\theoremstyle{plain}
\newtheorem{theorem}{Theorem}[section]
\newtheorem{proposition}[theorem]{Proposition}
\newtheorem{lemma}[theorem]{Lemma}
\newtheorem{corollary}[theorem]{Corollary}
\theoremstyle{definition}
\newtheorem{definition}[theorem]{Definition}
\newtheorem{assumption}[theorem]{Assumption}
\theoremstyle{remark}
\newtheorem{remark}[theorem]{Remark}
\newtheorem*{example}{Example}

\usepackage[textsize=tiny]{todonotes}

\title{Best-of-Both-Worlds Linear Contextual Bandits}

\renewcommand{\shorttitle}{Best-of-Both-Worlds Linear Contextual Bandits}

\author[1]{Masahiro Kato}
\author[2]{Shinji Ito}

\affil[1]{University of Tokyo, \href{mailto:mkato-csecon@g.ecc.u-tokyo.ac.jp}{mkato-csecon@g.ecc.u-tokyo.ac.jp}}
\affil[2]{NEC Corporation and RIKEN AIP, \href{mailto:i-shinji@nec.com}{i-shinji@nec.com}}

\begin{document}

\maketitle

\begin{abstract}
This study investigates the problem of $K$-armed linear contextual bandits, an instance of the multi-armed bandit problem, under an adversarial corruption. At each round, a decision-maker observes an independent and identically distributed context and then selects an arm based on the context and past observations. After selecting an arm, the decision-maker incurs a loss corresponding to the selected arm. The decision-maker aims to minimize the cumulative loss over the trial. The goal of this study is to develop a strategy that is effective in both stochastic and adversarial environments, with theoretical guarantees. We first formulate the problem by introducing a novel setting of bandits with adversarial corruption, referred to as the contextual adversarial regime with a self-bounding constraint. We assume linear models for the relationship between the loss and the context. Then, we propose a strategy that extends the {\tt RealLinExp3} by \citet{Neu2020} and the Follow-The-Regularized-Leader (FTRL). The regret of our proposed algorithm is shown to be upper-bounded by $O\left(\min\left\{\frac{(\log(T))^3}{\Delta_{*}} + \sqrt{\frac{C(\log(T))^3}{\Delta_{*}}},\ \ \sqrt{T}(\log(T))^2\right\}\right)$, where $T \in\mathbb{N}$ is the number of rounds, $\Delta_{*} > 0$ is the constant minimum gap between the best and suboptimal arms for any context, and $C\in[0, T] $ is an adversarial corruption parameter. This regret upper bound implies $O\left(\frac{(\log(T))^3}{\Delta_{*}}\right)$ in a stochastic environment and by $O\left( \sqrt{T}(\log(T))^2\right)$ in an adversarial environment. We refer to our strategy as the {\tt Best-of-Both-Worlds (BoBW) RealFTRL}, due to its theoretical guarantees in both stochastic and adversarial regimes.
\end{abstract}

\section{Introduction}
This study considers minimizing the cumulative regret in the multi-armed bandit (MAB) problem with contextual information. The MAB problem is a formulation of sequential decision-making. In this study, we develop an algorithm that utilizes side information called contextual information. We focus on linear contextual bandits and aim to design an algorithm that performs well in both stochastic and adversarial environments.

In our problem setting of contextual bandits, a decision-maker observes an independent and identically distributed (i.i.d.) context each round, draws an arm accordingly, and incurs a loss associated with the chosen arm. Additionally, we assume linear models between the loss and contexts, which is known as the linear contextual bandit problem. The contextual bandit problem is widely studied in fields such as sequential treatment allocation \citep{Tewari2017}, personalized recommendations \citep{Beygelzimer2011}, and online advertising \citep{Li2010}. Based on these demands, existing studies explore the methods. For example, \citet{AbeLong1999} studies linear contextual bandits. \citet{Li2021regret} provides lower bounds. There are numerous other studies in this field \citep{Chen2020efficient,Qin2022robuststochastic}.

The settings of linear contextual bandits are divided into stochastic, with fixed contextual and loss distributions, and adversarial environments, with fixed contexts but adversarially chosen losses\footnote{We can define adversarial linear contextual bandits in different ways. For example, there are studies that consider contextual bandits with adversarial contexts and fixed losses \citep{Chu2011contextual,AbbasiYadkori2011}. On the other hand, several studies address contextual bandits with adversarial contexts and adversarial losses \citep{KanadeSteinke2014,HazanKoren2016}. This study only focuses on contextual bandits with i.i.d. contexts and adversarial losses, which have been studied by \citet{Rakhlin2016} and \citet{Syrgkanis2016improved}.}. Most existing studies focus on algorithms for either stochastic \citep{AbeLong1999,Rusmevichientong2010,Chu2011contextual,AbbasiYadkori2011,Lattimore2017} or adversarial linear contextual bandits \citep{Neu2020}.

Thus, optimal algorithms typically differ between the stochastic and adversarial environments. However, a best-of-both-worlds framework exists, aiming for algorithms that are competent in both stochastic and adversarial environments \citep{BubeckSlivkins2012,SeldinSlivkins2014,AuerChiang2016,SeldinLugosi2017,Zimmert2019,Lee2021achieving}. Building on existing work, we propose a best-of-both-worlds algorithm for stochastic and adversarial linear contextual bandits.

\subsection{Main Contribution}
In Section~\ref{sec:prob}, we first introduce the setting of linear contextual bandits with adversarial corruption by defining the linear contextual adversarial regime with a self-bounding constraint. This setting is a generalization of the \emph{adversarial regime with a self-bounding constraint} proposed by \citet{Zimmert2019}. Under this regime, we bridge the stochastic and adversarial environments by an adversarial corruption parameter $C \geq 0$, where $C = 0$ corresponds to a stochastic environment and $C = T$ corresponds to an adversarial environment.

Then, in Section~\ref{sec:algorithm} inspired by the {\tt RealLinEXP3} proposed by \citet{Neu2020} for adversarial contexts, our algorithm uses the Follow-the-Regularized-Leader (FTRL) approach to adapt well to the stochastic environment. Our algorithm design also follows existing studies in best-of-both-worlds (BoBW) studies, such as \citet{ito2022nearly}. We refer to our algorithm as the {\tt BoBW-RealFTRL}.

In Section~\ref{sec:regretanalysis}, we show the upper bound of the {\tt BoBW-RealFTRL} as $O\left(\min\left\{\frac{D}{\Delta_{*}} + \sqrt{\frac{CD}{\Delta_{*}}}, \sqrt{\log(KT)TD}\right\}\right)$, where $D =  {K\log(T)}\left(\log(T) + d\log(K)\right)\log(KT)$, $T$ is the number of rounds, $d$ is the dimension of a context, and $K$ is the number of arms, when there exists a constant minimum gap $\Delta_*$ between the conditional expected rewards of the best and suboptimal arms for any context when we consider a stochastic environment. Note that this regret upper bound holds both for stochastic and adversarial environments. When there does not exist such a gap $\Delta_*$, we show that the regret upper bound is given as $O\left(\sqrt{\log(KT)TD}\right)$. Note that this regret upper bound also holds both for stochastic and adversarial environments, as well as the previous upper bound. Combining them, the regret upper bound is $O\left(\min\left\{\frac{D}{\Delta_{*}} + \sqrt{\frac{CD}{\Delta_{*}}}, \sqrt{\log(KT)TD}\right\}\right)$. Note that the regret upper bound under an adversarial environment can hold without any assumption on the existence of $\Delta_{*}$. Our regret upper bound is $O\left(\min\left\{\frac{(\log(T))^3}{\Delta_{*}} + \sqrt{\frac{C(\log(T))^3}{\Delta_{*}}},\ \ \sqrt{T}(\log(T))^2\right\}\right)$ when focusing on the order with respect to $T$. Furthermore, in a stochastic environment, the regret is upper bounded by $O\left(\frac{(\log(T))^3}{\Delta_{*}}\right)$, and in an adversarial environment, the regret is upper bound by $O\left(\sqrt{T}(\log(T))^2\right)$.

In summary, we contribute to the problem of linear contextual bandits by proposing a best-of-both-worlds strategy. 
Our study enhances the fields of linear contextual bandits and best-of-both-worlds algorithms.

\subsection{Related Work}
In adversarial bandits, the {\tt RealLinExp3}, the algorithm proposed by \citet{Neu2020}, yields $O\left(\log(T)\sqrt{KdT}\right)$. 
In Table~\ref{table:comparison}, we compare our regret upper bounds with the upper bounds of \citet{Neu2020}. 

Regret upper bounds in a stochastic setting are categorized into problem-dependent and problem-independent upper bounds, where the former utilizes some distributional information, such as the gap parameter $\Delta_*$, to bound the regret, while the latter does not. Additionally, problem-dependent regret upper bounds in the stochastic bandits depend on the margin condition characterized by a parameter $\alpha \in [0, +\infty]$ (for the detailed definition, see Remark~\ref{rem:margin}). Our case with $\Delta_*$ corresponds to a case with $\alpha = +\infty$. Note that in the adversarial bandits, the margin condition usually does not affect the upper bounds. \citet{Dani2008StochasticLO} proposes the {\tt ConfidenceBAll}, and \citet{AbbasiYadkori2011} proposes {\tt OFUL}. They both present upper bound with and without the assumption of the existence of $\Delta_*$\footnote{Regret upper bounds with the assumption of the existence of $\Delta_*$ are called problem-dependent.}. As mentioned above, the regret upper bound under the assumption of the existence of $\Delta_*$ corresponds to a case with $\alpha = +\infty$ in the margin condition. In contrast, \citet{Goldenshluger2013}, \citet{Wang2018j}, and \citet{Bastani2020} propose algorithm in a case with $\alpha = 1$. Furthermore, \citet{Li2021regret} propose the $\ell_1$-{\tt ConfidenceBall} based algorithm whose upper bound tightly depends on unknown $\alpha$.

There are several related studies for linear contextual bandits with adversarial corruption, including \citet{Lykouris2018}, \citet{Gupta2019}, \citet{zhao2021linear} and \citet{he2022nearly}. \citet{Lykouris2018}, \citet{Gupta2019}, and \citet{zhao2021linear} consider other corruption frameworks characterized by a constant $\widetilde{C} \in [0, T]$, which is different but related to our linear contextual adversarial regime with a self-bounding constraint. \citet{he2022nearly} uses another constant $\widetilde{C}^\dagger \in [0, T]$ different but closely related to $\widetilde{C}$. For the detailed definitions, see Remark~\ref{rem:lincontextual_corruption}. The essential difference between our and their settings is the existence of the gap $\Delta_*$. Furthermore, while our regret upper bound achieves the polylogarithmic order, those studies show roughly $\sqrt{T}$-order regret upper bounds. \citet{he2022nearly} presents $\widetilde{O}\left(d\sqrt{K} + d\widetilde{C}^\dagger \right)$ regret under an adversarial corruption characterized by a constant $\widetilde{C}^\dagger > 0$.

\begin{table}[t]
 \vspace{-0.5cm}
 \caption{Comparison of the regret. We compared the regret upper bound of our proposed {\tt BoBW-RealFTRL} with the {\tt RealLinExp3} \citep{Neu2020}.}
    \label{table:comparison}
    \centering
    \scalebox{1.0}[1.0]{
    \begin{tabular}{|l|c|c|}
    \hline
        & Regret  & Adversarial/Stochastic \\
    \hline
       Ours ({\tt BoBW-RealFTRL}) &  $O\left(\min\left\{\frac{D}{\Delta_{*}} + \sqrt{\frac{CD}{\Delta_{*}}}, \sqrt{\log(KT)TD}\right\}\right)$ & Both \\
        &  where $
        D =  {K\log(T)}\left(\log(T) + d\log(K)\right)\log(KT)$ &  \\
        & $O\left(\sqrt{\log(KT)TD}\right)$ & Adversarial \\
        & $O\left(\frac{D}{\Delta_{*}}\right)$ ($C = 0$) & Stochastic  \\
    \hline
    {\tt RealLinExp3} & $O\left(\log(T)\sqrt{KdT}\right)$ & Adversarial \\
    \hline
    \end{tabular}
    }
\end{table}

The use of the FTRL approach for adversarial linear bandits is also independently explored by \citet{liu2023bypassing} to relax the assumption used in \citet{Neu2020}. In addition to the difference in contributions, while our algorithm utilizes the Shannon entropy in the regularization of the FTRL, \citet{liu2023bypassing} employs the log-determinant barrier. We expect that combining these two methods will yield a BoBW algorithm with relaxed assumptions, and it is future work.

To establish our BoBW regret bounds,
we utilize the self-bounding technique~\citep{Zimmert2019,wei2018more},
which yields poly-logarithmic regret bounds in stochastic environments.
This is achieved by integrating regret upper bounds that are contingent on the arm-selection distributions $q_t$,
and a lower bound known as self-bound constraints.
The $q_t$-dependent regret bounds are obtained using FTRL with a negative-entropy regularizer,
which is also referred to as the \textit{exponential weight} method.
Our approach includes an entropy-adaptive update rule for learning rates,
originally developed for online learning in feedback graph contexts~\citep{ito2022nearly}.
This strategy has been proven effective in providing BoBW guarantees for exponential-weight-based algorithms across various sequential decision-making problems,
such as multi-armed bandits~\citep{jin2023improved},
partial monitoring~\citep{tsuchiya2023best},
linear bandits~\citep{kong2023best},
episodic Markov Decision Processes (MDPs)~\citep{dann2023best},
and sparse multi-armed bandits~\citep{tsuchiya2023stability}.
However,
a common limitation of these results,
stemming from the negative-entropy regularization,
is the additional $\log T$ factors in the regret bounds.
A promising future direction to mitigate this could be
exploring alternative regularizers like Tsallis entropy or logarithmic barriers.

\section{Problem Setting}
\label{sec:prob}
Suppose that there are $T$ rounds and $K$ arms. In each round $t\in[T] := \{1,2,\dots,T\}$, a decision-maker observes a context $X_t\in\mathcal{X}\subset \mathbb{R}^d$, where $\mathcal{X}$ is a context space. Then, the decision-maker chooses an arm $A_t\in[K] := \{1,2,\dots, K\}$ based on the context $X_t$ and past observations. Each arm $a\in[K]$ is linked to a loss $\ell_{t}(a, X_t)$, which depends on $X_t$ and round $t$. After choosing arm $A_t$ in round $t$, the decision-maker incurs the loss $\ell_{t}(A_t, X_t)$. Our goal is to minimize the cumulative loss $\sum^T_{t=1}\ell_{t}(A_t, X_t)$. We introduce the setting in more detail in the following part.

\paragraph{Contextual distribution.}
Let a distribution of $X_t$ be $\mathcal{D}$, which is invariant across $t\in[T]$. We also assume that $\mathcal{D}$ is known to the decision-maker.
\begin{assumption}[Contextual distribution]
    The context $X_t$ is an i.i.d. random variable, whose distribution $\mathcal{D}$ is \textbf{known} to the decision-maker, and the covariance matrix $\Sigma = \mathbb{E}\left[X_t X^\top_t\right]$ is positive definite with its smallest eigenvalue $\lambda_{\min} > 0$. 
\end{assumption}

\paragraph{Loss.}
In each round $t\in[T]$, a context is sampled from $\mathcal{D}$ and  the environment chooses $\ell_t(\cdot, X_t)$ based on the past observations $\mathcal{F}_{t-1} = (X_1, A_1, \ell_1(A_1, X_1), X_2, \dots, X_{t-1}, A_{t-1}, \ell_{t-1}(A_{t-1}, X_{t-1}))$. We consider a general framework where $\ell_t$ is generated in both stochastic and adversarial ways. See Section~\ref{sec:context_adversarial} for details.

\paragraph{Policy.}
We refer to a function that determines the arm draw as a policy. 
Let $\Pi$ be a set of all possible policies $\pi:\mathcal{X}\to \mathcal{P} := \Big\{ u = (u_1\ u_2\ \dots\ u_K)^\top \in [0, 1]^K \mid \sum^K_{k=1}u_k = 1 \Big\}$. Let $\pi(a\mid x)$ be the $a$-th element of $\pi(x)$. 
The goal of the decision-maker is to minimize the cumulative loss $\sum^T_{t=1}\ell_t(A_t, X_t)$ incurred through $T$ rounds by learning a policy $\pi\in\Pi$.

\paragraph{Procedure in a trial.} In each round of a trial, the decision-maker first observes a context and then chooses an action based on the context and past observations obtained until the round. 
Specifically, we consider sequential decision-making with the following steps in each round $t\in[T]$:
\begin{enumerate}
    \item The environment decides $\big(\ell_t(1, X_t), \ell_t(2, X_t), \dots, \ell_t(K, X_t)\big)$ based on $\mathcal{F}_{t-1}$. 
    \item The decision-maker observes the context $X_t$, which is generated from a known distribution $\mathcal{D}$.
    \item Based on the context $X_t$, the decision-maker chooses a policy $\pi_t(X_t) \in \mathcal{P}$.
    \item The decision-maker chooses action $A_t \in [K]$ with probability $\pi_t(a\mid X_t)$.
    \item The decision-maker incurs the loss $\ell_t(A_t, X_t)$. 
\end{enumerate}
The goal of the decision-maker is to choose actions in a way that the total loss is as small as possible. 

\subsection{Linear Contextual Bandits}
This study assumes linear models between $\ell_t(A_t, X_t)$ and $X_t$ as follows.
\begin{assumption}[Linear models] For each $\ell_t(a, X_t)$, the following holds:
    \begin{align*}
    \ell_t(a, X_t) = x^\top_t \theta_{t}(a) + \varepsilon_{t}(a),
\end{align*}
where $\theta_{t}(a)$ is a $d$-dimensional parameter, and $\varepsilon_{t}(a)$ is the error term independent of the sequence $\{X_t\}_{t\in[T]}$. 
\end{assumption}

For linear models and variables, we make the following assumptions.
\begin{assumption}[Bounded variables]
We assume the following:
\begin{enumerate}
    \item There exists an universal constant $C_{\mathcal{X}} > 0$ such that for each $x \in\mathcal{X}$, $\|x\|_2 \leq C_{\mathcal{X}}$ holds. 
    \item There exists an universal constant $C_{\Theta} > 0$ such that for each $\theta \in\Theta$, $\|\theta\|_2 \leq C_{\Theta}$ holds. 
    \item There exists an universal constant $C_{\mathcal{E}} > 0$ such that $|\varepsilon_{t}(a)| \leq C_{\mathcal{E}}$ holds. 
\end{enumerate}    
\end{assumption}
Under this assumption, there exists $C_{\ell} := C(C_{\mathcal{X}}, C_{\Theta}, C_{\mathcal{E}})$ such that for all $\ell_t(a, X_t)$, the following holds for each $a\in[K]$ and $x\in\mathcal{X}$:
\begin{align*} 
|\ell_t(a, x)| \leq C_{\ell}.
\end{align*}

\subsection{Regret}
This section provides the definition of the regret, a relative measure of the cumulative loss. We evaluate the performance of the decision or policy of the decision-maker by using regret. 
Let $\mathcal{R}$ be a set of all possible $\rho:\mathcal{X}\to [K]$.
The quality of a decision by the decision-maker is measured by its total expected regret, defined as
\begin{align*}
    R_T = \max_{\rho\in\mathcal{R}}\mathbb{E}\left[\sum^T_{t=1}\Big\{ \ell_t(A_t, X_t) - \ell_t(\rho(X_t), X_t)\Big\}\right] = \max_{\rho\in\mathcal{R}}\mathbb{E}\left[\sum^T_{t=1}\Big\langle X_t, \theta_{t}(A_t) - \theta_{t}(\rho(X_t))\Big\rangle\right],
\end{align*}
where the expectation is taken over the randomness of policies of the decision-maker, as well as the sequence of random contexts, $\{X_t\}_{t\in[T]}$, and losses, $\{\ell_t(\cdot, X_t)\}_{t\in[T]}$.  

Let $X_0$ be an i.i.d. random variable from the distribution of $X_t$. Then, because $X_t$ is an i.i.d. random variable from $\mathcal{D}$, we have
\begin{align*}
    \mathbb{E}\left[\sum^T_{t=1}\Big\langle X_t, \theta_{t}(\rho(X_t))\Big\rangle\right] =  \mathbb{E}\left[\sum^T_{t=1}\Big\langle X_0, \theta_{t}(\rho(X_0))\Big\rangle\right] \geq \mathbb{E}\left[\min_{a\in[K]}\sum^T_{t=1}\Big\langle X_0, \mathbb{E}\left[\theta_{t}(a)\right]\Big\rangle\right]. 
\end{align*}

Based on this inequality, we define an optimal policy $a^*$ as
\begin{align*}
    &a^*_T(x) = \argmin_{a\in[K]} \sum^T_{t=1}\Big\langle x, \mathbb{E}\left[\theta_{t}(a)\right]\Big\rangle.
\end{align*}
Then, we have
\begin{align*}
    R_T \leq \mathbb{E}\left[\sum^T_{t=1}\Big\langle X_t, \theta_{t}(A_t) - \theta_{t}(a^*_T(X_t))\Big\rangle\right].
\end{align*}

\citet{Neu2020} refers to $\rho$ as linear-classifier policies, while $\pi_t$ is called stochastic policies. In our study, decision-makers compare their stochastic policies $\pi_t$ to the optimal linear-classifier policy $a^*$ using the regret.

\subsection{Linear Contextual Adversarial Regime with a Self-Bounding Constraint}
\label{sec:context_adversarial}

Then, we define the framework of a \emph{linear contextual adversarial regime with a self-bounding constraint}, which is a generalization of adversarial and stochastic bandits. 

\begin{definition}[Linear contextual adversarial regime with a self-bounding constraint]
\label{def:bounding}
We say that an environment is in an adversarial regime with a $(\Delta_*, C, T)$ self-bounding constraint for some $\Delta_*, C > 0$ if $R_T$ is lower bounded as
\begin{align*}
    R_T \geq \Delta_* \cdot \mathbb{E}\left[\sum^T_{t=1}\Big( 1 - \pi_t(a^*_T(X_0)\mid X_0) \Big)\right] - C.
\end{align*}
\end{definition}

The contextual adversarial regime with a self-bounding constraint includes several important settings. Among them, we raise linear contextual bandits in stochastic bandits and adversarial bandits below.

\begin{example}[Linear contextual bandits in stochastic bandits.]
\label{exm:stochastic}
    In stochastic bandits, the bandit models are fixed; that is, $\big(X_t, \ell_t(1, X_t), \dots, \ell_t(K, X_t)\big)$ are generated from a fixed distribution $P_0$. Let $\theta_{1}(a) = \cdots \theta_{T}(a) = \theta_0(a)$. Note that when considering stochastic bandits, we have $\mathbb{E}\left[\theta_{t}(a)\right] = \theta_0(a)$ and
\begin{align*}
    a^*_T(x) = \argmin_{a\in[K]} \sum^T_{t=1}\Big\langle x, \mathbb{E}\left[\theta_{t}(a)\right]\Big\rangle = \argmin_{a\in[K]} \Big\langle x, \theta_{0}(a)\Big\rangle\qquad \forall x \in \mathcal{X}.
\end{align*}
Let $a^*_T(x)$ be $a^*_0(x)$. 
    
    In this setting, we assume that for each $P_0$, there exist positive constraints $\Delta_*$ such that for all $x \in \mathcal{X}$, 
    \begin{align}
    \label{eq:gap}
    \min_{b\neq a^*_0(x)}\Big\langle x, \theta_0(b)\Big\rangle -  \Big\langle x, \theta_{0}(a^*_0(x))\Big\rangle \geq \Delta_*.
    \end{align}
    Then, the regret can be lower bounded as $R_T \geq \Delta_* \cdot \mathbb{E}\left[\sum^T_{t=1}\Big( 1 - \pi_t(a^*_0(X_0)\mid X_0) \Big)\right]$ (See Appendix~\ref{appdx:example+stochastic}). 
\end{example}

\begin{example}[Linear contextual bandits in adversarial bandits]
In adversarial bandits, we do not assume any data-generating process for the $\ell_t(a, X_t)$, and the loss is decided to increase the regret based on the past observations $\mathcal{F}_{t-1}$.
\end{example}

\begin{remark}[Margin conditions]
\label{rem:margin}
In linear contextual bandits, we often employ the \emph{margin condition} to characterize the difficulty of the problem instance. The margin condition is defined as follows \citep{Li2021regret}: there exist $\Delta_*$, $C_1$, $a^*$, and $\alpha \in [0, +\infty]$, such that for $h\in\left[C_1\sqrt{\frac{\log(d)}{T}}, \Delta_*\right]$, 
\[\mathbb{P}\left(\Big\langle X_t, \theta_{t}(a^*)\Big\rangle \leq \max_{b \neq a^*}\Big\langle X_t, \theta_{t}(b)\Big\rangle + h\right)\leq \frac{1}{2}\left(\frac{h}{\Delta_*}\right)^\alpha.\]
Our definition of the linear contextual adversarial regime with a self-bounding constraint corresponds to a case with $\alpha = \infty$. Extending our results to more general $\alpha$ is a future work.
\end{remark}

\begin{remark}[Linear contextual bandits with corruption in existing studies] 
\label{rem:lincontextual_corruption}
\citet{Lykouris2018}, \citet{Gupta2019}, \citet{zhao2021linear}, and \citet{he2022nearly} propose another definition of linear contextual contextual bandits with corruption. In their work, instead of our defined $\ell_t(a, X_t)$, they define a loss as
\begin{align*}
    \widetilde{\ell}_t(a, X_t) = \ell_t(a, X_t) + \widetilde{c}_t(a),
\end{align*}
where $\widetilde{c}_t(a)$ is an adversarial corruption term. For simplicity, let $\widetilde{c}_t(a) \in [-1, 1]$. In \citet{zhao2021linear}, the degree of the corruption is determined by $\widetilde{C} \in [0, T]$ defined as $\widetilde{C} = \sum^T_{t=1}\max_{a\in[K]}|c_t(a)|$. In \citet{Lykouris2018}, \citet{Gupta2019}, and \citet{he2022nearly}, the corruption level is determined by another parameter $\widetilde{C}^\dagger \in [0, T]$ defined as $\sum^T_{t = 1}|c_t(A_t)|$. Here, $\widetilde{C} \geq \widetilde{C}^\dagger$ holds. Note that the adversarial corruption depends on $A_t$ in \citet{he2022nearly}, while the adversarial corruption is determined irrelevant to $A_t$ in \citet{Lykouris2018}, \citet{Gupta2019}, and \citet{he2022nearly}.
Unlike ours, they do not assume the existence of $\Delta_*$ defined in \eqref{eq:gap}. In this sense, our results and their results are complementary.
\end{remark}

\section{Algorithm}
\label{sec:algorithm}
This section provides an algorithm for our defined problem. Our proposed algorithm is a generalization of the {\tt RealLinEXP3} algorithm proposed by \citet{Neu2020}. We extend the method by employing the Follow-The-Regularized-Leader (FTRL) approach with round-varying arm-drawing probabilities. Our design of the algorithm is also motivated by existing studies about Best-of-Both-Worlds (BoBW) algorithms in different Multi-Armed Bandit (MAB) problems, such as \citet{ito2022nearly}.

In our setting, we first observe a context and then draw an arm based on that context. We consider stochastically drawing an arm. Therefore, in designing the algorithm, our interest lies in appropriately defining the arm-drawing probability. In the FTRL approach, we define this probability by utilizing an unbiased estimator of the loss function.

We refer to our proposed algorithm as {\tt BoBW-RealFTRL} because it modifies the {\tt RealLinEXP3} for a best-of-both-worlds algorithm using the FTRL framework. The pseudo-code is shown in Algorithm~\ref{alg}. In the following part, we explain the algorithm.

\paragraph{Unbiased loss estimator.}
For each $a\in[K]$, let us define an estimator of regression parameters as 
\begin{align*}
    \widehat{\theta}_{t}(a) := \Sigma^{\dagger}_{t, a}\mathbbm{1}[A_t = a]X_t\ell(A_t, X_t),
\end{align*}
where $\Sigma^{\dagger}_{t, a}$ is an estimator of $\mathbb{E}\left[\mathbbm{1}[A_t = a]X^\top_0 X_0\right]^{-1}$. 
Then, the loss can be estimated as
\begin{align*}
    \widehat{\ell}_{t}(a, x) = \left\langle x, \widehat{\theta}_{t}(a) \right\rangle.
\end{align*}

In analysis of adversarial bandits, the bias of $\widehat{\ell}_{t}(a, x)$ plays an important role. 
If $\Sigma^{\dagger}_{t, a} = \mathbb{E}\left[\mathbbm{1}[A_t = a]X^\top_0 X_0\mid \mathcal{F}_{t-1}\right]^{-1}$, then this loss estimator is unbiased  for $x^\top\theta_0(a)$ because 
\begin{align*}
    &\mathbb{E}\left[ \widehat{\ell}_{t}(a, x)\mid \mathcal{F}_{t-1}\right] = x\mathbb{E}\left[ \widehat{\ell}_{t}(a, x)\mid \mathcal{F}_{t-1}\right] = x\Sigma^{\dagger}_{t, a}\mathbb{E}\left[\mathbbm{1}[A_t = a]X_t\ell(A_t, X_t) \mid \mathcal{F}_{t-1}\right]\\
    &= x\Sigma^{\dagger}_{t, a}\mathbb{E}\left[\mathbbm{1}[A_t = a]X_t\left\{X^\top_t \theta_0(a) + \varepsilon_t(a) \right\} \mid \mathcal{F}_{t-1}\right] = x^\top\theta_0(a).
\end{align*}
Note that in our algorithm, $\Sigma^{\dagger}_{t, a}$ is just an estimator of $\mathbb{E}\left[\mathbbm{1}[A_t = a]X^\top_0 X_0\mid \mathcal{F}_{t-1}\right]^{-1}$, and  $\Sigma^{\dagger}_{t, a} = \mathbb{E}\left[\mathbbm{1}[A_t = a]X^\top_0 X_0\mid \mathcal{F}_{t-1}\right]^{-1}$ does not hold in general. Therefore, $\widehat{\ell}_{t}(a, x)$ is not unbiased. However, we can show that the bias can be ignored because it is sufficiently small to evaluate the regret in depth. 
We also define a vector of loss estimators as $\widehat{\ell}_{t}(x) = \left(\widehat{\ell}_{t}(1, x)\ \ \widehat{\ell}_{t}(2, x)\ \ \cdots\ \ \widehat{\ell}_{t}(K, x)\right)^\top$. 

\paragraph{Estimation of $\mathbb{E}\left[\mathbbm{1}[A_t = a]X^\top_0 X_0\mid \mathcal{F}_{t-1}\right]^{-1}$.} Our remaining task is to estimate $\mathbb{E}\left[\mathbbm{1}[A_t = a]X^\top_0 X_0\mid \mathcal{F}_{t-1}\right]^{-1}$. The difficulty of this task stems from the dependency on $A_t$, which varies across rounds. To address this issue, we consider Matrix Geometric Resampling (MGR) proposed by \citet{Neu2020}. 

The MGR assumes that we have access to the distribution $\mathcal{D}$ of $X_t$ and estimates $\mathbb{E}\left[\mathbbm{1}[A_t = a]X^\top_0 X_0\right]^{-1}$ by using simulations. We introduce the algorithm in Algorithm~\ref{alg2}. 

In Algorithm~\ref{alg2}, we define $W_{k, a}$ for which $\mathbb{E}[W_{k, a} \mid \mathcal{F}_{t-1}] = \Sigma_{t, a}$ holds. Here, from the independence of the context $X(k)$ from each other, we also have $\mathbb{E}[V_{k, a} \mid \mathcal{F}_{t-1}] = \mathbb{E}\left[ \prod^k_{j=1} \big( I - \delta W_{j, a}\big) \mid \mathcal{F}_{t-1}\right] = (I - \delta \Sigma_{t, a})^k$. Therefore, $\widehat{\Sigma}^\dagger_{t, a}$ works as a good estimator of $\Sigma^{-1}_{t, a}$ on expectations when $M_t = \infty$ because
\begin{align*}
    \mathbb{E}\left[\widehat{\Sigma}^\dagger_{t, a} \mid  \mathcal{F}_{t-1}\right] = \delta  I + \delta \sum^\infty_{k=1}(I - \delta \Sigma_{t, a})^k = \delta\sum^\infty_{k=0}(I - \delta \Sigma_{t, a})^k = \delta(\delta \Sigma^{-1}_{t, a})^{-1} = \Sigma^{-1}_{t, a}. 
\end{align*}
holds. 

In implementation, $M_t$ is finite, and we introduce an approximation error of $\Sigma^{-1}_{t, a}$ with finite $M_t$ in Lemma~\ref{cor:neu2020}.

\paragraph{Our proposed algorithm: {\tt BoBW-RealFTRL}.}
Then, we define our policy, called the {\tt BoBW-RealFTRL}, as
\begin{align}
\label{eq:policy}
    \pi_t(X_t) := (1-\gamma_t)q_t(X_t) + \frac{\gamma_t}{K}\iota,
\end{align}
where $\iota$ is a $K$-dimensional vector $\iota = (1\ 1\ \cdots\  1)^\top$, 
\begin{align}
    &q_t(x) \in \argmin_{q\in \Pi}\left\{ \sum^{t-1}_{s=1}\left\langle \widehat{\ell}_t(x), q(x)\right\rangle + \psi_t(q(x))\right\}\ \ \mathrm{for}\ t \geq 2,\quad q_1(x) := (1/K\ 1/K\ \cdots\ 1/K)^\top,\nonumber\\
    \label{eq:entropy}
    &\psi_t(q(x)) := -\beta_t H(q(x)),\quad H(q(x)) := \sum_{a\in[K]}q(a\mid x)\log \left(\frac{1}{q(a\mid x)}\right),\\
    \label{eq:entropy2}
    &\beta_{t+1} := \beta_t + \frac{\beta_1}{\sqrt{1 + \big(\log(K)\big)^{-1}\sum^t_{s=1}H\big(q_s(X_s)\big)}},\quad \beta_1 := \omega \sqrt{\frac{\log(KdT)}{\log(K)}},\\
    \label{eq:entropy3}
    &\omega := C_{\ell}C_{\mathcal{X}},\quad \gamma_t := \frac{K}{2\delta \lambda_{\min}\beta_t}\log(T),\\
    &M_t := 2\beta_t - 1,\ \ \ \mathrm{and}\ \ \ \delta := \frac{1}{2C_{\ell}C_{\mathcal{X}}}.\nonumber
\end{align}
This algorithm is an extension of the {\tt RealLinEXP3} proposed by \citep{Neu2020}
and the FTRL. In the studies of BoBW algorithms, the FTRL-based algorithms are often employed, and our algorithm is connected to the literature. 


\begin{algorithm}[tb]
    \caption{{\tt BoBW-RealFTRL}.}
    \label{alg}
    \begin{algorithmic}
    \STATE {\bfseries Parameter:} Learning rate $\eta_1, \eta_2,\dots, \eta_T > 0$, exploration parameter $\gamma\in(0, 1)$. 
    \STATE {\bfseries Initialization:} Set $\theta_{0}(a) = 0$ for all $a\in[K]$.
    \FOR{$t=1,\dots, T$}
    \STATE Observe $X_t$. 
    \STATE Draw $A_t \in [K]$ following the policy $\pi_t(X_t) := (1-\gamma_t)q_t(X_t) + \frac{\gamma_t}{K}\iota$ defined in \eqref{eq:policy}.
    \STATE Observe the loss $\ell_t(A_t, X_t)$.
    \STATE Compute $\widehat{\theta}_t(a)$ for all $a\in[K]$. 
    \ENDFOR
\end{algorithmic}
\end{algorithm}

\begin{algorithm}[tb]
    \caption{Matrix Geometric Resampling \citep{Neu2020}.}
    \label{alg2}
    \begin{algorithmic}
    \STATE {\bfseries Input:} Context distribution $\mathcal{D}$, policy $\pi_t$, action $a\in[K]$. 
    \FOR{$k=1,\dots, M_t$}
    \STATE Draw $X(k) \sim \mathcal{D}$ and $V(k) \sim \pi_t(\cdot \mid X(k))$.
    \STATE Compute $W_{k, a} = \mathbbm{1}[V(k) = a]X(k)X^\top(k)$. 
    \STATE Compute $V_{k, a} = \prod^k_{j=1}(I - \delta W_{k, a})$.
    \ENDFOR
    \STATE {\bfseries Return:} $\widehat{\Sigma}^\dagger_{t, a} = \delta  I + \delta \sum^{M_t}_{k=1}V_{k, a}$.  
\end{algorithmic}
\end{algorithm}

\section{Regret Analysis}
\label{sec:regretanalysis}
This section provides upper bounds for the regret of our proposed {\tt BoBW-RealFTRL} algorithm. 

For notational simplicity, let us denote $a^*_T$ by $a^*$. 
To derive upper bounds, we define the following quantities:
\begin{align*}
    Q(a^*\mid x) = \sum^T_{t=1}\Big\{ 1 - \pi_t\big(a^*(x)\mid x\big) \Big\},\qquad \overline{Q}(a^*) = \mathbb{E}\left[ Q(a^*\mid X_0) \right].
\end{align*}

Then, we show the following upper bound, which holds for general cases such as adversarial and stochastic environments. We show the proof in Sections~\ref{sec:proof1} and \ref{sec:proof2}.
\begin{theorem}[General regret bounds]
\label{thm:regret_bound}
If the environment generates losses under the contextual adversarial regime with a self-bounding constraint (Definition~\ref{def:bounding}), the {\tt BoBW-RealFTRL} with $\widehat{\Sigma}^\dagger_{t, a}$ incurs the total regret 
    \begin{align*}
        R_T \leq O\left(\left( \frac{K\log(T)}{\beta_1}\left(\frac{\log(T)}{ \delta \lambda_{\min}\log(K)} + d\right) + \beta_1\sqrt{\log(K)}\right)\sqrt{\log(KT)}\max\Big\{\overline{Q}^{1/2}(a^*), 1\Big\}\right).
    \end{align*}
\end{theorem}
For each situation, such as adversarial environments and linear contextual adversarial regimes with a self-bounding constraint, we derive a specific upper bound.

First, from $\overline{Q}(a^*) \leq T$, the following regret bound holds without any assumptions on the loss; that is, it holds for an adversarial environment. 
\begin{corollary}
    Assume the same conditions in Theorem~\ref{thm:regret_bound}. Then, under an adversarial environment, the regret satisfies 
    \[R_T = O\left(\left( \frac{K\log(T)}{\beta_1}\left(\frac{\log(T)}{ \delta \lambda_{\min}\log(K)} + d\right) + \beta_1\sqrt{\log(K)}\right)\sqrt{\log(KT)}\sqrt{T}\right);\] that is, from $\beta_1 = \omega \sqrt{ {K\log(T)}\left(\frac{\log(T)}{ \delta \lambda_{\min}\log(K)} + d\right)}$, 
    \[R_T = {O}\left(\log(KT)\sqrt{ {K\log(K)T\log(T)}\left(\frac{\log(T)}{ \delta \lambda_{\min}\log(K)} + d\right)}\right)\] holds.
\end{corollary} 

Furthermore, we derive a regret bound under the linear contextual adversarial regime with a self-bounding constraint. 
\begin{corollary}[Regret bounds under the linear contextual adversarial regime with a self-bounding constraint]
\label{cor:adversarial_regime}
Suppose that the same conditions in Theorem~\ref{thm:regret_bound} hold. Then, under the contextual adversarial regime with self-bounding constraints, the regret satisfies
    \begin{align*}
        R_T &= O\Bigg(\left\{ \frac{K\log(T)}{\beta_1}\left(\frac{\log(T)}{ \delta \lambda_{\min}\log(K)} + d\right) + \beta_1\sqrt{\log(K)}\right\}^2\log(KT)/ \Delta_{*}\\
        &\ \ \ \ \ \ \ \ \ \ \ \ \ \ \ \ \ \ \ \ \ \ \ \ \ \ \ \ \ \ + \sqrt{{C\left\{\frac{K\log(T)}{\beta_1}\left(\frac{\log(T)}{ \delta \lambda_{\min}\log(K)} + d\right) + \beta_1\sqrt{\log(K)}\right\}^2}\log(KT)/{\Delta_{*}}}\Bigg);
    \end{align*} that is, from $\beta_1 = \omega \sqrt{ {K\log(T)}\left(\frac{\log(T)}{ \delta \lambda_{\min}\log(K)} + d\right)}$, \[R_T = O\left(\frac{D}{\Delta_{*}} + \sqrt{\frac{CD}{\Delta_{*}}}\right)\] holds, where
    \begin{align*}
        D =  {K\log(K)\log(T)}\left(\frac{\log(T)}{ \delta \lambda_{\min}\log(K)} + d\right)\log(KT).
    \end{align*}
\end{corollary}
The result in Corollary~\ref{cor:adversarial_regime} implies $R_T = O\left(\frac{(\log(T))^3}{\Delta_{*}} + \sqrt{\frac{C(\log(T))^3}{\Delta_{*}}}\right)$.
\begin{proof}
    From the definition of the contextual adversarial regime with a self-bounding constraint, we have
    \begin{align*}
        R_T \geq \Delta_* \cdot \mathbb{E}\left[\sum^T_{t=1}\Big( 1 - \pi_t(a^*(X_0)\mid X_0) \Big)\right] - C =  \Delta_* \cdot \overline{Q}(a^*) - C.
    \end{align*}
    Therefore, from Lemma~\ref{lem:regret_basic_bounds}, for any $\lambda > 0$, we have
    \begin{align*}
        &R_T = (1+\lambda)R_T - \lambda R_T\\
        &= (1+\lambda)O\left(c\sqrt{\log(KT)}\sqrt{\sum^T_{t=1}\mathbb{E}\left[H(q_t(X_0))\right]}\right) - \lambda R_T\\
        &\leq (1+\lambda)O\left(c\sqrt{\log(KT)}\sqrt{\sum^T_{t=1}\mathbb{E}\left[H(q_t(X_0))\right]}\right) - \lambda \Delta_* \cdot \overline{Q}(a^*) + \lambda C,
    \end{align*}
    where 
    \[c = \left( \frac{K\log(T)}{\beta_1}\left(\frac{\log(T)}{ \delta \lambda_{\min}\log(K)} + d\right) + \beta_1\sqrt{\log(K)}\right).\]
    Here, as well as the proof of Theorem~\ref{thm:regret_bound}, from Lemma~\ref{lem:bound_H}, if $Q(a^*\mid x) \leq e$, we have $\sum^T_{t=1}H(q_t(x)) \leq e\log(KT)$ and otherwise, we have $\sum^T_{t=1}H(q_t(x)) \leq Q(a^*\mid x)\log(KT)$. Hence, we have $\sum^T_{t=1}H(q_t(x))  \leq \log(KT)\max\{e, Q(a^*\mid x)\}$. Here, to upper bound $R_T$, it is enough to only consider a case with $Q(a^*\mid x) \geq e$, and we obtain
    \begin{align*}
        &R_T \leq (1+\lambda)O\left(c\sqrt{\log(KT)}\sqrt{\overline{Q}(a^*)\log(KT)}\right) - \lambda \Delta_* \cdot \overline{Q}(a^*) + \lambda C\leq \frac{O\left(\Big\{(1+\lambda)c\Big\}^2\sqrt{\log(KT)}\right)}{2\lambda\Delta_*} + \lambda \Delta_*.
    \end{align*}
    where the second inequality follows from $a\sqrt{b} - \frac{c}{2}b \leq \frac{a^2}{c^2}$ holds for any $a,b,c > 0$. By choosing
    \[\lambda = \sqrt{\frac{c^2\log(KT)}{\Delta_*}\Big/ \left(\frac{c^2\log(KT)}{\Delta_*} + 2C\right)}.\]
    Then, we obtain $R_T = O\Bigg(c^2\log(KT)/ \Delta_{*} + \sqrt{{Cc^2}\log(KT)/{\Delta_{*}}}\Bigg)$. 
\end{proof}


In the following sections, we show the proof procedure of Theorem~\ref{thm:regret_bound}.

\subsection{Preliminaries for the Proof of Theorem~\ref{thm:regret_bound}}
\label{sec:proof1}
Let $X_0$ be a sample from the context distribution $\mathcal{D}$ independent of $\mathcal{F}_T$. Let $D_t(p, q)$ denote the Bregman divergence of $p. q\in\Pi$ with respect to $\psi_t$; that is,
\begin{align*}
    D_t(p, q) = \psi_t(p) - \psi_t(q) - \Big\langle \nabla \psi_t(q), p - q \Big\rangle.
\end{align*}
Let us define $\pi^*\in \Pi$ as $\pi^*(a^*(x)\mid x) = 1$ and $\pi^*(a\mid x) = 0$ for all $a\in[K]\backslash \{a^*(x)\}$. 

Then, the following lemma holds. The proof is shown in Appendix~\ref{appdx:lem_baisc}
\begin{lemma}
\label{lem:basic}
    If $A_t$ is chosen as our proposed method, the regret is bounded by
    \begin{align*}
        R_T &\leq \mathbb{E}\Bigg[\sum^T_{t=1}\Big\{\gamma_t + \left\langle \widehat{\ell}_t(X_0), q_t(X_0) - q_{t+1}(X_0) \right\rangle\\
        &\ \ \ \ \ \ \ \ \ \ \ \ \ \ \ \ \ \ \ \ - D_t(q_{t+1}(X_0), q_t(X_0)) + \psi_t(q_{t+1}(X_0)) - \psi_{t+1}(q_{t+1}(X_0))\Big\}\\
        &\ \ \ \ \ \ \ \ \ \ \ \ \ \ \ \ \ \ \ \ + \psi_{T+1}(\pi^*(X_0)) - \psi_1(q_1(X_0))\Bigg]  + 2\sum^T_{t=1}\max_{a\in[K]}\left| \mathbb{E}\left[\langle X_t, \theta_{t}(a) - \widehat{\theta}_{t}(a) \rangle\right] \right|.
    \end{align*}
\end{lemma}
To show Lemma~\ref{lem:basic}, we use the following proposition from \citet{Neu2020}.
\begin{proposition}
Suppose that $\pi_t\in\mathcal{F}_{t-1}$ and that $\mathbb{E}\left[\widehat{\theta}_{t, a}|\mathcal{F}_{t-1}\right] = \theta_{t, a}$ for all $t, a$ hold. Then, the following holds:
    \begin{align*}
        \mathbb{E}\left[\sum^T_{t=1}\sum_{a\in[K]}\Big(\pi_t(a\mid X_t) - \pi^*(a\mid X_t)\Big)\Big\langle X_t, \theta_{t,a}\Big\rangle\right] = \mathbb{E}\left[\sum^T_{t=1}\sum_{a\in[K]}\Big(\pi_t(a\mid X_0) - \pi^*(a\mid X_0)\Big)\Big\langle X, \widehat{\theta}_{t,a}\Big\rangle\right].
    \end{align*}
\end{proposition}
This proposition plays an important role throughout this study. 

\paragraph{Bounding the stability term.}
For the stability term $\left\langle \widehat{\ell}_t(X_0), q_t(X_0) - q_{t+1}(X_0) \right\rangle - D_t(q_{t+1}(X_0), q_t(X_0))$, we use the following proposition from \citet{ito2022nearly}.
\begin{proposition}[From Lemma~8 in \citet{ito2022nearly}]
\label{prp:lem8}
    If $\psi_t$ is given as \eqref{eq:entropy}, for any $\ell: \mathcal{X} \to \mathbb{R}^K$ and $p, q \in\Pi$, we have
    \begin{align*}
        &\Big\langle \ell(x), p(x) - q(x) \Big\rangle - D_t(q(x), p(x)) \leq \beta_t\sum_{a\in[K]}p(a\mid x)\xi\left(\frac{\ell(a, x)}{\beta_t}\right).
    \end{align*}
    for any $x\in\mathcal{X}$, where $\xi(x) := \exp(-x) + x - 1$. 
\end{proposition}
For $\widehat{\ell}(a, x)$, if $\frac{\widehat{\ell}(a, x)}{\beta_t} \geq -1$ holds, then Proposition~\ref{prp:lem8} implies 
\begin{align*}
    &\Big\langle \widehat{\ell}_t(x), q_t(x) - q_{t+1}(x) \Big\rangle - D_t(q_{t+1}(x), q_t(x))\leq \frac{1}{\beta_t}\sum_{a\in[K]}\pi_t(a\mid x)\widehat{\ell}^2_t(a, x).
\end{align*}
For the RHS, we apply the following proposition from \citet{Neu2020}. 
\begin{proposition}[From Lemma~6 in \citet{Neu2020}]
\label{prp:lemma6_neu}
   For each $t\in[T]$, our strategy satisfies
    \begin{align*}
        \mathbb{E}\left[\sum_{a\in[K]}\pi_t(a\mid X_0)\widehat{\ell}^2_t(a, X_0)\mid \mathcal{F}_{t-1}\right] \leq 3Kd. 
    \end{align*}
\end{proposition}

\paragraph{Estimation error of the design matrix.}
Next, we bound $\sum^T_{t=1}\max_{a\in[K]}\left| \mathbb{E}\left[\langle X_t, \theta_{t}(a) - \widehat{\theta}_{t}(a) \rangle\right] \right|$. An upper bound of $\sum^T_{t=1}\max_{a\in[K]}\left| \mathbb{E}\left[\langle X_t, \theta_{t}(a) - \widehat{\theta}_{t}(a) \rangle\right] \right|$ is given as the following lemma.
\begin{lemma}
    \label{cor:neu2020}
    We have $\left| \mathbb{E}\left[\langle X_t, \theta_{t}(a) - \widehat{\theta}_{t}(a) \rangle\right] \right| \leq C_{\mathcal{X}} C_{\Theta} / T$.
\end{lemma}
\begin{proof}[Proof of Lemma~\ref{cor:neu2020}]
From Lemma~5 in \citet{Neu2020}, we have $\left| \mathbb{E}\left[\langle X_t, \theta_{t}(a) - \widehat{\theta}_{t}(a) \rangle\right] \right| \leq C_{\mathcal{X}} C_{\Theta}  \exp\left( - \frac{\gamma_t \delta}{K}\lambda_{\min}M_t\right)$. Then, we have
\begin{align*}
    &\exp\left( - \frac{\gamma_t \delta}{K}\lambda_{\min}M_t\right) = \exp\left(-\frac{K\log(T)}{\delta \lambda_{\min}\cdot 2\beta_t}\frac{\delta\lambda_{\min}}{K}M_t\right)\\
    &\leq \exp\left(-\frac{K\log(T)}{\delta \lambda_{\min}\cdot \left(2\beta_t - 1\right)}\frac{\delta\lambda_{\min}}{K}M_t\right)= \exp\left(-\log(T)\right) = \frac{1}{T},
\end{align*}
where recall that we defined $M_t = \beta_t - 1 $
\end{proof}

\subsection{Proof of Theorem~\ref{thm:regret_bound}}
\label{sec:proof2}
Then, we show the following lemma. The proof is shown in Appendix~\ref{appdx:proof_thm9}. 
\begin{lemma}
\label{thm:basic_bounds}
The regret for the {\tt BoBW-RealFTRL} with $\widehat{\Sigma}^\dagger_{t, a}$ is bounded as
    \begin{align*}
        R_T &\leq \mathbb{E}\left[\sum^T_{t=1}\Bigg\{\gamma_t + \frac{3Kd}{\beta_t} + \left(\beta_{t+1} - \beta_t\right)H(q_{t+1}(X_0))\Bigg\}\right] + \beta_1\log(K) + 2 C_{\mathcal{X}} C_{\Theta}.
    \end{align*}
\end{lemma}

From this result, we obtain the following lemma. We provide the proof in Appendix~\ref{appdx:regret_basic_bounds}
\begin{lemma}
\label{lem:regret_basic_bounds}
    Assume the conditions in Theorem~\ref{thm:basic_bounds}. Suppose that $\beta_t$ and $\gamma_t$ satisfy \eqref{eq:entropy2} and \eqref{eq:entropy3}. Then, we have
    \begin{align*}
        R_T \leq \overline{c} \sqrt{\mathbb{E}\left[\sum^T_{t=1}H(q_t(X_0))\right]} + 2 C_{\mathcal{X}} C_{\Theta},
    \end{align*}
    where $\overline{c} = O\left( \frac{K\log(T)}{\beta_1}\left(\frac{\log(T)}{ \delta \lambda_{\min}\log(K)} + d\right) + \beta_1\sqrt{\log(K)}\right)$. 
\end{lemma}

Next, we consider bounding $\sum^T_{t=1}H(q_t(x))$ by $Q(a^*\mid x)$ as shown in the following lemma.
\begin{lemma}
\label{lem:bound_H}
For any $a^*:\mathcal{X}\to[K]$, the following holds:
\begin{align*}
    \sum^T_{t=1}H(q_t(x)) \leq Q(a^*\mid x) \log\left(\frac{eKT}{Q(a^*\mid x)}\right),
\end{align*}
where $e$ is Napier's constant. 
\end{lemma}

By using the above lemmas and propositions, we prove Theorem~\ref{thm:regret_bound}.
\begin{proof}[Proof of Theorem~\ref{thm:regret_bound}]
    From Lemma~\ref{lem:bound_H}, if $Q(a^*\mid x) \leq e$, we have $\sum^T_{t=1}H(q_t(x)) \leq e\log(KT)$ and otherwise, we have $\sum^T_{t=1}H(q_t(x)) \leq Q(a^*\mid x)\log(KT)$. Hence, we have $\sum^T_{t=1}H(q_t(x))  \leq \log(KT)\max\{e, Q(a^*\mid x)\}$. From Lemma~\ref{lem:regret_basic_bounds}, we have
    \begin{align*}
        R_T &\leq \overline{c} \sqrt{\sum^T_{t=1}\mathbb{E}\left[H(q_t(X_0))\right]} + 2 C_{\mathcal{X}} C_{\Theta}\\
        &\leq O\left(\left( \frac{K\log(T)}{\beta_1}\left(\frac{\log(T)}{ \delta \lambda_{\min}\log(K)} + d\right) + \beta_1\sqrt{\log(K)} \right)\sqrt{\log(KT)}\max\Big\{\overline{Q}^{1/2}, 1\Big\}\right).
    \end{align*}
\end{proof}

\section{Conclusion}
\label{sec:conclusion}
We developed a BoBW algorithm for linear contextual bandits. Our proposed algorithm is based on the FTRL approach. In our theoretical analysis, we show that the upper bounds of the proposed algorithm are given as $O\left(\min\left\{\frac{D}{\Delta_{*}} + \sqrt{\frac{CD}{\Delta_{*}}}, \sqrt{\log(KT)TD}\right\}\right)$, where $D =  {K\log(T)}\left(\log(T) + d\log(K)\right)\log(KT)$. This regret upper bound implies $O\left(\min\left\{\sqrt{\frac{TD}{\Delta_{*}}}, \sqrt{\log(KT)TD}\right\}\right)$ regret in an adversarial environment and $O\left(\frac{D}{\Delta_{*}}\right)$ regret in an adversarial environment and $O\left(\frac{D}{\Delta_{*}}\right)$ regret in a stochastic environment. This result also implies $O\left(\frac{(\log(T))^3}{\Delta_{*}}\right)$ regret in a stochastic regime and $O\left(\sqrt{T}(\log(T))^2\right)$ regret in an adversarial regime with respect to $T$.

There are four directions for future work in this study. The first direction is to develop an algorithm that does not require a contextual distribution while maintaining the BoBW property. We expect this extension can be accomplished by applying our proposed method to a method proposed by \citet{liu2023bypassing}, based on the FTRL approach with the log-determinant barrier. We note that standard linear contextual bandits in a stochastic environment do not require the contextual distribution to be known, but it is required for dealing with an adversarial environment.

The second direction is to provide lower bounds in our adversarial regimes.
In existing studies, \citet{Li2021regret} provides a general upper bound that holds for a high-dimensional setting with various margin conditions. We can incorporate such results to derive a lower bound in our problem setting.

The third extension is to develop an algorithm that works for linear contextual bandits without assuming a specific minimum gap constant $\Delta_{*}$. To address this issue, we might use the margin condition to generalize the minimum gap assumption. Lastly, tightening our regret upper bound is also an open problem.

\bibliography{arXiv.bbl} 
\bibliographystyle{tmlr} 

\appendix

\clearpage

\section{Details of Example~\ref{exm:stochastic}}
\label{appdx:example+stochastic}
When $\min_{b\neq a^*_0(x)}\Big\langle x, \theta_0(b)\Big\rangle -  \Big\langle x, \theta_{0}(a^*_0(x))\Big\rangle \geq \Delta_*$ holds for all $x\in\mathcal{X}$, we have
\begin{align*}
    R_T &= \mathbb{E}\left[\sum^T_{t=1}\ell_t(A_t, X_t) - \sum^T_{t=1}\ell_t(a^*, X_t)\right]\\
    &= \mathbb{E}\left[\sum^T_{t=1}\sum_{a\in[X]}X_t\Big(\theta^*_a - \theta^*_{a^*_t}\Big)\pi_t(a\mid X_t)\right]\\
    &= \mathbb{E}\left[\sum^T_{t=1}\sum_{a\in[X]}X_t\Big(\theta^*_a - \theta^*_{a^*_t}\Big)\pi_t(a\mid X_t)\mathbbm{1}\left[\min_{b\neq a^*_t}\Big\langle X_{t}, \theta^*_b\Big\rangle - \Big\langle X_t, \theta^*_{a^*_t}\Big\rangle \leq  \Delta_*\right]\right]\\
    &\ \ \ + \mathbb{E}\left[\sum^T_{t=1}\sum_{a\in[X]}X_t\Big(\theta^*_a - \theta^*_{a^*_t}\Big)\pi_t(a\mid X_t)\mathbbm{1}\left[\min_{b\neq a^*_t}\Big\langle X_{t}, \theta^*_b\Big\rangle - \Big\langle X_t, \theta^*_{a^*_t}\Big\rangle >  \Delta_*\right]\right]\\
    &\geq \mathbb{E}\left[\sum^T_{t=1}\sum_{a\in[X]}X_t\Big(\theta^*_a - \theta^*_{a^*_t}\Big)\pi_t(a\mid X_t)\mathbbm{1}\left[\min_{b\neq a^*_t}\Big\langle X_{t}, \theta^*_b\Big\rangle - \Big\langle X_t, \theta^*_{a^*_t}\Big\rangle >  \Delta_*\right]\right]\\
    &\geq \Delta_* \cdot \sum^T_{t=1}\mathbb{E}\left[Q^2_t(a^*(X_t))\right].
\end{align*}

\section{Proof of Lemma~\ref{lem:basic}}
\label{appdx:lem_baisc}
Let us define
\begin{align*}
    \widehat{R}_T(x) := \sum^T_{t=1}\sum_{a\in[K]} \Big(\pi_t(a\mid x) - \pi^*(a\mid x)\Big)\left\langle x, \widehat{\theta}_{t, a}\right\rangle.
\end{align*}

Then, the following holds:
\begin{align*}
    R_T \leq \mathbb{E}\left[\widehat{R}_T(X_0)\right] + 2\sum^T_{t=1}\max_{a\in[K]}\left| \mathbb{E}\left[\langle X_t, \theta_{t, a} - \widehat{\theta}_{t, a} \rangle\right] \right|.
\end{align*}

Then, we prove Lemma~\ref{lem:basic} as follows. 
\begin{proof}[Proof of Lemma~\ref{lem:basic}]
From the definition of the algorithm, we have
\begin{align}
    R_T((a^*_t)_{t\in[T]}) &= \mathbb{E}\left[\sum^T_{t=1}\ell_t(A_t, X_t) - \sum^T_{t=1}\ell_t(a^*_t, X)\right]\nonumber\\
    &= \mathbb{E}\left[\sum^T_{t=1}\left\langle \ell_t(X_t), \pi_t(X_t) - \pi^*(X_t)\right\rangle \right]\nonumber\\
    &= \mathbb{E}\left[\sum^T_{t=1}\left\langle \ell_t(X_t), q_t(X_t) - \pi^*(X_t)\right\rangle + \sum^T_{t=1}\gamma_t \left\langle \ell_t(X_t), \mu_U - q_t(X_t)\right\rangle \right]\nonumber\\
    &\leq \mathbb{E}\left[\sum^T_{t=1}\Big\langle \ell_t(X_t), q_t(X_t) - \pi^*(X_t)\Big\rangle + \sum^T_{t=1}\gamma_t \right]\nonumber\\
    &= \mathbb{E}\left[\sum^T_{t=1}\Big\langle \ell_t(X_0), q_t(X_0) - \pi^*(X_0)\Big\rangle + \sum^T_{t=1}\gamma_t \right]\nonumber\\
    &= \mathbb{E}\left[\sum^T_{t=1}\left\langle \widehat{\ell}_t(X_0), q_t (X_0)- \pi^*(X_0)\right\rangle + \sum^T_{t=1}\gamma_t \right]\nonumber\\
    &\ \ \ + \mathbb{E}\left[\sum^T_{t=1}\left\langle \ell_{t}(X_0) - \widehat{\ell}_t(X_0), q_t(X_0) - \pi^*(X_0)\right\rangle \right]\nonumber\\
    \label{eq:target3}
    &\leq \mathbb{E}\left[\sum^T_{t=1}\left\langle \widehat{\ell}_t(X_0), q_t (X_0)- \pi^*(X_0)\right\rangle + \sum^T_{t=1}\gamma_t \right] + 2\sum^T_{t=1}\max_{a\in[K]}\left| \mathbb{E}\left[\Big\langle X_t, \theta_{t, a} - \widehat{\theta}_{t, a} \Big\rangle\right] \right|.
\end{align}
Then, from the definitions of $q_t$, for each $x\in\mathcal{X}$, we also have
    \begin{align*}
         &\sum^T_{t=1}\left\langle \widehat{\ell}_t(x), \pi^*(x)\right\rangle + \psi_{T+1}(\pi^*(x))\\
         &\geq \sum^T_{t=1}\left\langle \widehat{\ell}_t(x), q_{T+1}(x) \right\rangle + \psi_{T+1}(q_{T+1}(x)) - \psi_{T+1}(q_{T+1}(x))\\
         &\ \ \ + \psi_{T+1}(\pi^*(x)) - \langle \nabla \psi_t(q_{T+1}(x)), \pi^*(x) - q_{T+1}(x)\rangle\\
         &= \sum^T_{t=1}\left\langle \widehat{\ell}_t(x), q_{T+1}(x) \right\rangle + \psi_{T+1}(q_{T+1}(x)) +D_{T+1}(\pi^*(x), q_{T+1}(x)),
    \end{align*}
    where we used that $\langle \nabla \psi_t(q_{T+1}(x)), \pi^*(x) - q_{T+1}(x)\rangle \geq 0$ holds for a convex function $\psi_t$. Then, it holds that
    \begin{align*}
         &\sum^T_{t=1}\left\langle \widehat{\ell}_t(x), \pi^*(x) \right\rangle + \psi_{T+1}(\pi^*(x))\\
         &\geq \sum^T_{t=1}\left\langle \widehat{\ell}_t(x), q_{T+1}(x) \right\rangle +D_{T+1}(\pi^*(x), q_{T+1}(x)) + \psi_{T+1}(q_{T+1}(x))\\
         &\geq \sum^T_{t=1}\left\langle \widehat{\ell}_t(x), q_{T+1}(x) \right\rangle + \psi_T(q_T(x))\\
         &\ \ \ + D_T(q_{T+1}(x), q_T(x)) + D_{T+1}(\pi^*(x), q_{T+1}(x)) - \psi_T(q_{T+1}(x)) + \psi_{T+1}(q_{T+1}(x))\\
         &= \sum^{T-1}_{t=1}\left\langle \widehat{\ell}_t(x), q_{T+1}(x) \right\rangle + \psi_T(q_T(x))\\
         &\ \ \ + \left\langle \widehat{\ell}_T(x), q_{T+1}(x) \right\rangle + D_T(q_{T+1}(x), q_T(x)) + D_{T+1}(\pi^*(x), q_{T+1}(x)) - \psi_T(q_{T+1}(x)) + \psi_{T+1}(q_{T+1}(x))\\
         &\geq \sum^{T-1}_{t=1}\left\langle \widehat{\ell}_t(x), q_{T}(x) \right\rangle + \psi_T(q_T(x))\\
         &\ \ \ + \left\langle \widehat{\ell}_T(x), q_{T+1}(x) \right\rangle + D_T(q_{T+1}(x), q_T(x)) + D_{T+1}(\pi^*(x), q_{T+1}(x)) - \psi_T(q_{T+1}(x)) + \psi_{T+1}(q_{T+1}(x))\\
         &\geq \sum^T_{t=1}\left\langle \widehat{\ell}_t(x), q_{t+1}(x) \right\rangle + \sum^T_{t=1}D_t(q_{t+1}(x), q_t(x)) - \sum^T_{t=1}\Big\{\psi_t(q_{t+1}(x)) - \psi_{t+1}(q_{t+1}(x))\Big\} + \psi_1(q_1(x)).
    \end{align*}
    Therefore, we have
    \begin{align*}
         &\sum^T_{t=1}\left\langle \widehat{\ell}_t(x), q_t(x) - \pi^*(x) \right\rangle\\
         &\leq \sum^T_{t=1}\left\{\left\langle \widehat{\ell}_t(x), q_t(x) - q_{t+1}(x) \right\rangle - D_t(q_{t+1}(x), q_t(x)) + \psi_t(q_{t+1}(x)) - \psi_{t+1}(q_{t+1}(x))\right\}\\
         &\ \ \ \ \ + \psi_{T+1}(\pi^*(x)) - \psi_1(q_1(x)).
    \end{align*}
    Combining this with \eqref{eq:target3}, we obtain 
     \begin{align*}
        &R_T((a^*_t)_{t\in[T]})\\
        &\leq \mathbb{E}\Bigg[\sum^T_{t=1}\left\{\left\langle \widehat{\ell}_t(X_0), q_t(X_0) - q_{t+1}(X_0) \right\rangle - D_t(q_{t+1}(x), q_t(X_0)) + \psi_t(q_{t+1}(X_0)) - \psi_{t+1}(q_{t+1}(X_0))\right\}\\
         &\ \ \ \ \ + \psi_{T+1}(\pi^*(X_0)) - \psi_1(q_1(X_0)) + \sum^T_{t=1}\gamma_t \Bigg] + 2\sum^T_{t=1}\max_{a\in[K]}\left| \mathbb{E}\left[\Big\langle X_t, \theta_{t, a} - \widehat{\theta}_{t, a} \Big\rangle\right] \right|.
    \end{align*}
\end{proof}

\section{Proof of Lemma~\ref{thm:basic_bounds}}
\label{appdx:proof_thm9}
\begin{proof}[Proof of Lemma~\ref{thm:basic_bounds}]
From Lemma~\ref{lem:basic}, we have
\begin{align*}
    R_T &\leq \mathbb{E}\Bigg[\sum^T_{t=1}\Big(\gamma_t + \left\langle \widehat{\ell}_t(X_0, d), \pi_t(X_0) - q_{t+1}(X_0) \right\rangle - D_t(q_{t+1}(X_0), \pi_t(X_0))\\
    &\ \ \ \ \ \ \ \ \ \ \ \ \ \ \ \ \ \ \ \ \ \ \ \ \ \ \ \ \ \ \ \ \ \ \ \ \ \ \ + \psi_t(q_{t+1}(X_0)) - \psi_{t+1}(q_{t+1}(X_0))\Big)+ \psi_{T+1}(\pi^*(X_0)) - \psi_1(q_1(x))\Bigg]\\
    &+ 2\sum^T_{t=1}\max_{a\in[K]}\left| \mathbb{E}\left[\langle X_t, \theta_{t, a} - \widehat{\theta}_{t, a} \rangle\right] \right|.
\end{align*}

First, we show
\begin{align}
\label{eq:target_lemma49}
    &\mathbb{E}\left[\left\langle \widehat{\ell}_t(X_0), \pi_t(X_0) - q_{t+1}(X_0) \right\rangle - D_t(q_{t+1}(X_0), \pi_t(X_0))\right] \leq  \frac{3Kd}{\beta_t}. 
\end{align}

To show this, we confirm $\frac{\widehat{\ell}_t(a, x)}{\beta_t} \geq -1$, which is necessary to derive an upper bound from Proposition~\ref{prp:lem8}. We have
\begin{align*}
    \frac{1}{\beta_t} \cdot \Big\langle X_0, \widehat{\theta}_t(a) \Big\rangle &= \frac{1}{\beta_t} \cdot X^\top_0 \widehat{\Sigma}^\dagger_{t, a}X_t \Big\langle X_t, \theta_{t, a}\Big\rangle \mathbbm{1}[A_t = a] \geq - \frac{C_{\ell}}{\beta_t} \cdot \left| X^\top_0 \widehat{\Sigma}^\dagger_{t, a} X_t \right|\\
    &\geq - \frac{1}{\beta_t} C_{\ell}C_{\mathcal{X}} \left\| \widehat{\Sigma}^\dagger_{t, a} \right\|_{\mathrm{op}} \geq - \frac{1}{\beta_t} C_{\ell}C_{\mathcal{X}}\delta \left( 1 + \sum^{M_t}_{k=1}\|V_{k, a}\|_{\mathrm{op}} \right)= - \frac{1}{2\beta_t} (M_t + 1),
\end{align*}
where we used that $\delta = \frac{1}{2C_{\ell}C_{\mathcal{X}}}$. 
Here, recall that we defined $M_t$ as $2\beta_t - 1$. Therefore, $\frac{\widehat{\ell}_t(a, x)}{\beta_t} = -1$ holds. Then, we have
\begin{align*}
    &\left\langle \widehat{\ell}_t(x), \pi_t(x) - q_{t+1}(x) \right\rangle - D_t(q_{t+1}(x), \pi_t(x))\\
    &\leq \beta_t\sum_{a\in[K]}\pi_t(a\mid x)\xi\left(\frac{\widehat{\ell}_t(a, x)}{\beta_t}\right) \leq  \frac{1}{\beta_t}\sum_{a\in[K]}\pi_t(a\mid x)\widehat{\ell}^2_t(a, x). 
\end{align*}
Then, from Proposition~\ref{prp:lemma6_neu}, we have \eqref{eq:target_lemma49}. 

From $\psi_t(q(x)) = -\beta_t H(q(x))$, we have
\begin{align*}
&\sum^T_{t=1}\left(\psi_t(q_{t+1}(x)) - \psi_{t+1}(q_{t+1}(x))\right)+ \psi_{T+1}(\pi^*(x)) - \psi_1(q_1(x))\\
&\leq \sum^T_{t=1}\left(\beta_{t+1} - \beta_t\right)H(q_{t+1}(x)) + \beta_1\log(K).
\end{align*}

From Lemma~\ref{cor:neu2020}, we have
\begin{align*}
    \sum^T_{t=1}\max_{a\in[K]}\left| \mathbb{E}\left[\langle X_t, \theta_{t, a} - \widehat{\theta}_{t, a} \rangle\right] \right| \leq \sum^T_{t=1}C_{\mathcal{X}} C_{\Theta}  \frac{1}{\sqrt{T}} = C_{\mathcal{X}} C_{\Theta}\sqrt{T}. 
\end{align*}
\end{proof}

\section{Proof of Lemma~\ref{lem:regret_basic_bounds}}
\label{appdx:regret_basic_bounds}
\begin{proof}[Proof of Lemam~\ref{lem:regret_basic_bounds}]
Firstly, we note that the following equality holds:
\begin{align*}
    &\mathbb{E}\left[\sum^T_{t=1}\left(\beta_{t+1} - \beta_t\right)H(q_{t+1}(X_{t+1}))\right]\\
    &=\mathbb{E}\left[\sum^T_{t=1}\left(\beta_{t+1} - \beta_t\right)\mathbb{E}\left[H(q_{t+1}(X_{t+1}))\mid \mathcal{F}_{t}\right]\right]\\
    &=\mathbb{E}\left[\sum^T_{t=1}\left(\beta_{t+1} - \beta_t\right)\mathbb{E}\left[H(q_{t+1}(X_{0}))\mid \mathcal{F}_{t}\right]\right]\\    &= \mathbb{E}\left[\sum^T_{t=1}\left(\beta_{t+1} - \beta_t\right)H(q_{t+1}(X_0))\right]
\end{align*}
    We show the following two inequalities:
    \begin{align}
    \label{eq:target1}
        \sum^T_{t=1}\Bigg(\gamma_t + \frac{3Kd}{\beta_t}\Bigg)  &= O\left( {\frac{K\log(T)}{\beta_1}\left(\frac{\log(T)}{ \delta \lambda_{\min}\log(K)} + d\right)}\sqrt{\sum^T_{s=1}H\big(q_{t+1}(X_s)\big)}\right)\\
         \label{eq:target2}
        \sum^T_{t=1}\left(\beta_{t+1} - \beta_t\right)H(q_{t+1}(X_{t+1})) &= O\left( \beta_1\sqrt{\log(K)} \sqrt{\sum^T_{t=1}H(q_t(X_{t})} \right).
    \end{align}

    First, we show \eqref{eq:target1}. From $\gamma_t = \frac{K}{4\delta \lambda_{\min}\beta_t}\log(T)$, we obtain
    \begin{align*}
    \sum^T_{t=1}\Bigg(\gamma_t + \frac{3Kd}{\beta_t}\Bigg) = \sum^T_{t=1}\Bigg(\frac{K}{2\delta \lambda_{\min}\beta_t}\log(T) + \frac{3Kd}{\beta_t}\Bigg) = \Bigg(\frac{K}{2\frac{1}{2C_{\ell}C_{\mathcal{X}}}\delta \lambda_{\min}}\log(T) + {3Kd}\Bigg)\sum^T_{t=1}\frac{1}{\beta_t}.
    \end{align*}
    From $\beta_{t+1} = \beta_t + \frac{\beta_1}{\sqrt{1 + \big(\log(K)\big)^{-1}\sum^t_{s=1}H\big(q_{s}(X_s)\big)}}$, we obtain 
    \begin{align*}
        \beta_t = \beta_1 + \sum^{t-1}_{u=1}\frac{\beta_1}{\sqrt{1 + \big(\log(K)\big)^{-1}\sum^u_{s=1}H\big(q_{s}(X_s)\big)}} \geq \frac{t\beta_1}{\sqrt{1 + \big(\log(K)\big)^{-1}\sum^t_{s=1}H\big(q_{s}(X_s)\big)}}.
    \end{align*}
    Therefore, we have
    \begin{align*}
        \sum^T_{t=1}\frac{1}{\beta_t} \leq \sum^T_{t=1}\frac{\sqrt{1 + \big(\log(K)\big)^{-1}\sum^t_{s=1}H\big(q_{s}(X_s)\big)}}{t\beta_1}\leq \frac{1 + \log(T)}{\beta_1}\sqrt{1 + \big(\log(K)\big)^{-1}\sum^T_{s=1}H\big(q_{s}(X_s)\big)}.
    \end{align*}
    By using $H\big(q_{1}(x)\big) = \log(K)$, we obtain
    \begin{align*}
    \sum^T_{t=1}\Bigg(\gamma_t + \frac{3Kd}{\beta_t}\Bigg) = O\left( {\frac{K\log(T)}{\beta_1}\left(\frac{\log(T)}{ \delta \lambda_{\min}\log(K)} + d\right)}\sqrt{\sum^T_{s=1}H\big(q_{t+1}(X_s)\big)}\right).
    \end{align*}

    Next, we show \eqref{eq:target2}. From the definitions of $\beta_t$ and $\gamma_t$, we have
    \begin{align*}
        &\sum^T_{t=1}\left(\beta_{t+1} - \beta_t\right)H(q_{t+1}(X_{t+1})) = \sum^T_{t=1}\frac{\beta_1}{\sqrt{1 + \big(\log(K)\big)^{-1}\sum^t_{s=1}H\big(q_{s}(X_s)\big)}}H(q_{t+1}(X_{t+1}))\\
        &= 2\beta_1\sqrt{\log(K)}\sum^T_{t=1}\frac{H\big(q_{t+1}(X_{t+1})\big)}{\sqrt{\log(K) + \sum^t_{s=1}H\big(q_{s}(X_s)\big)} + \sqrt{\log(K) + \sum^t_{s=1}H\big(q_{s}(X_s)\big)}}\\
        &\leq 2\beta_1\sqrt{\log(K)}\sum^T_{t=1}\frac{H\big(q_{t+1}(X_{t+1})\big)}{\sqrt{\log(K) + \sum^{t+1}_{s=1}H\big(q_{s}(X_s)\big)} + \sqrt{\log(K) + \sum^t_{s=1}H\big(q_{s}(X_s)\big)}}\\
        &\leq 2\beta_1\sqrt{\log(K)}\sum^T_{t=1}\frac{H\big(q_{t+1}(X_{t+1})\big)}{\sqrt{\sum^{t+1}_{s=1}H\big(q_{s}(X_s)\big)} + \sqrt{\sum^t_{s=1}H\big(q_{s}(X_s)\big)}}\\
        &= 2\beta_1\sqrt{\log(K)}\sum^T_{t=1}\frac{H(q_{t+1}(X_{t+1}))}{H(q_{t+1}(X_{t+1}))}\left\{{\sqrt{\sum^{t+1}_{s=1}H\big(q_{s}(X_s)\big)} - \sqrt{\sum^t_{s=1}H\big(q_{s}(X_s)\big)}}\right\}\\
        &= 2\beta_1\sqrt{\log(K)}\sum^T_{t=1}\left\{{\sqrt{\sum^{t+1}_{s=1}H\big(q_{s}(X_s)\big)} - \sqrt{\sum^t_{s=1}H\big(q_{s}(X_s)\big)}}\right\}\\
        &= 2\beta_1\sqrt{\log(K)}\left\{{\sqrt{\sum^{T+1}_{s=1}H\big(q_{s}(X_s)\big)} - \sqrt{H\big(q_{1}(X_1)\big)}}\right\}\\
        &\leq 2\beta_1\sqrt{\log(K)}\sqrt{\sum^{T}_{s=1}H\big(q_{s}(X_s)\big)},
    \end{align*}
    where we used $\sqrt{H(q_{T+1}(X_{T+1}))} \leq \sqrt{H(q_1(X_1))}$. 
    
Inequalities \eqref{eq:target1} and \eqref{eq:target2} combined with the inequality in Theorem~\ref{thm:basic_bounds} yield 
\begin{align*}
    R_T &\leq \mathbb{E}\left[\sum^T_{t=1}\Bigg\{\gamma_t + \frac{3Kd}{\beta_t} + \left(\beta_{t+1} - \beta_t\right)H(q_{t+1}(X_0))\Bigg\}\right] + \beta_1\log(K) + 2 C_{\mathcal{X}} C_{\Theta}\sqrt{T}\\
        &= \mathbb{E}\left[\sum^T_{t=1}\Bigg\{\gamma_t + \frac{3Kd}{\beta_t} + \left(\beta_{t+1} - \beta_t\right)H(q_{t+1}(X_{t+1}))\Bigg\}\right]+ \beta_1\log(K) + 2 C_{\mathcal{X}} C_{\Theta}\sqrt{T}\\
        &= \mathbb{E}\left[\sum^T_{t=1}\left\{O\left( \frac{K\log(T)\big(\log(T) + \delta \lambda_{\min}d\big)}{\beta_1 \delta \lambda_{\min}\log(K)}\sqrt{\sum^T_{s=1}H\big(q_{t+1}(X_s)\big)}\right) + O\left( \beta_1\sqrt{\log(K)} \sqrt{\sum^T_{t=1}H(q_t(X_{t})} \right)\right\}\right]\\
        &\ \ \ \ \ \ + \beta_1\log(K) + 2 C_{\mathcal{X}} C_{\Theta}\sqrt{T}\\
        &= \sum^T_{t=1}\left\{O\left( \frac{K\log(T)\big(\log(T) + \delta \lambda_{\min}d\big)}{\beta_1 \delta \lambda_{\min}\log(K)}\sqrt{\sum^T_{s=1}\mathbb{E}\left[H\big(q_{t+1}(X_0)\big)\right]\big)}\right) + O\left( \beta_1\sqrt{\log(K)} \sqrt{\sum^T_{t=1}\mathbb{E}\left[H(q_t(X_{t})\right]} \right)\right\}\\
        &\ \ \ \ \ \ + \beta_1\log(K) + 2 C_{\mathcal{X}} C_{\Theta}\sqrt{T}.
\end{align*}
Thus, we obtain the regret bound in Lemma~\ref{lem:regret_basic_bounds}. 
\end{proof}

\end{document}